\documentclass{article}

\usepackage[preprint]{ewrl_2026}

\usepackage[utf8]{inputenc} 
\usepackage[T1]{fontenc}    
\usepackage{hyperref}       
\usepackage{url}            
\usepackage{booktabs}       
\usepackage{amsfonts}       
\usepackage{nicefrac}       
\usepackage{microtype}      
\usepackage{xcolor}         
\usepackage{setspace}
\usepackage{a4wide}
\usepackage{amsfonts}
\usepackage{lipsum}
\usepackage[ruled,vlined,linesnumbered]{algorithm2e}
\usepackage{subcaption} 
\usepackage{graphicx}  

\usepackage{xcolor}  
\usepackage{listings}  
\usepackage{enumitem}
\usepackage{amsthm}
\usepackage{amssymb}
\usepackage{amsmath}
\usepackage{xifthen}
\usepackage{xparse}
\usepackage{dsfont}
\usepackage{mathrsfs}

\newtheorem{theorem}{Theorem}
\newtheorem{definition}[theorem]{Definition}

\newtheorem{lemma}[theorem]{Lemma}
\newtheorem{corollary}[theorem]{Corollary}
\newtheorem{remark}[theorem]{Remark}

\newcommand{\M}{\mathcal{M}}  
\newcommand{\Ss}{\mathcal{S}} 
\newcommand{\A}{\mathcal{A}} 
\newcommand{\U}{\mathcal{U}} 
\newcommand{\D}{\mathcal{D}} 

\newcommand{\Regret}{\operatorname{Regret}}

\newcommand{\argmax}{\operatorname{argmax}}

\newcommand{\lr}[1]{\left (#1\right)}

\newcommand{\lrs}[1]{\left [#1 \right]}

\newcommand{\lrc}[1]{\left \{#1\right\}}

\newcommand{\lra}[1]{\left |#1\right|}

\newcommand{\lri}[1]{\left\langle#1\right\rangle}

\newcommand{\lrv}[1]{\left\lVert #1 \right\rVert}

\newcommand{\R}{\mathbb R}

\newcommand{\N}{\mathbb N}

\NewDocumentCommand{\E}{o}{\mathbb E\IfValueT{#1}{\lrs{#1}}}

\NewDocumentCommand{\1}{o}{\mathds 1{\IfValueT{#1}{\lr{#1}}}}

\let\P\undefined
\NewDocumentCommand{\P}{o}{\mathbb P{\IfValueT{#1}{\lr{#1}}}}

\newcommand{\HH}{\mathcal H}

\newcommand{\YY}{\mathcal{Y}}

\DeclareMathOperator{\V}{Var}
\NewDocumentCommand{\Var}{o}{\V\IfValueT{#1}{\lrs{#1}}}

\title{Randomized Least Squares Value Iteration itself is Joint Differentially Private}

\author{
Haiyang Lu$^{1}$ \quad
Pratik Gajane$^{1}$ \quad
Shaojie Bai$^{2}$ \quad
Mohammad Sadegh Talebi$^{3}$ \\[0.45em]
{\small $^{1}$Laboratoire d'Informatique Fondamentale d'Orléans (LIFO), Université d'Orléans} \\
{\small $^{2}$College of Control Science and Engineering, Zhejiang University} \\
{\small $^{3}$Department of Computer Science, University of Copenhagen} \\[0.35em]
{\footnotesize
\texttt{\{haiyang.lu, pratik.gajane\}@univ-orleans.fr}} \\
{\footnotesize
\texttt{white.shaojie@gmail.com}, \texttt{sadegh.talebi@di.ku.dk}}
}

\begin{document}

\maketitle

\begin{abstract}
  As reinforcement learning (RL) increasingly applies to sensitive domains, such as health care and recommendation systems, privacy-preserving techniques have become essential to protect users' sensitive information. We investigate privacy-preserving RL under an episodic setting, focusing on algorithms based on randomized exploration, such as Randomized Least Squares Value Iteration (RLSVI). The overall goal is to study how randomized exploration interacts with the injected noise required by privacy mechanisms. In this work, we show a new privacy analysis that characterizes how the noise in RLSVI set for exploration simultaneously provides privacy protection. Specifically, we prove that RLSVI is $(\varepsilon(\delta),\delta)$-joint differentially private in tabular MDP as is with $\varepsilon(\delta) = \frac{2AK}{H^2\log(2HSA)} + 2\sqrt{\frac{2AK\log(1/\delta)}{H^2\log(2HSA)}}$, where $S$ and $A$ are the number of states and actions respectively, $H$ is the length of an episode and $K$ is the number of episodes.
\end{abstract}

\section{Introduction}
Reinforcement learning (RL) is a framework for sequential decision-making in which an agent learns to maximize its expected cumulative reward in an unknown environment through trial and error. Unlike supervised and unsupervised learning, which are typically trained on fixed datasets, the agent in online RL learns through interaction with the environment by adjusting its action policy based on the feedback received, making it necessary to balance exploration and exploitation \citep{sutton2018_ReinforcementLearningIntroduction}. In some domains, such as health care and recommendation systems, the agent continuously receives feedback from users that necessarily involves their personal information to improve its performance. Take a medical treatment scenario for example, a clinician recommends a medication to a patient, observes its effects, and updates the medication choice accordingly based on the patient's response. During the learning process, sensitive information of the patient, such as their health status, is accessed by the clinician (i.e., the agent in RL context) and causes privacy concerns. As a result, privacy-preserving techniques have become essential in RL to protect users' data \citep{lei2023_NewChallengesReinforcement}.

A popular notion of privacy is
differential privacy (DP) \citep{dwork2006calibrating}, which has been widely adopted across a wide range of applications since its introduction. The core objective of DP is to make the output of an algorithm statistically indistinguishable when any single individual’s data is changed, which addresses the challenge of learning useful information about a population while learning nothing about an individual \citep{dwork2014algorithmic}. DP has become increasingly popular nowadays in designing private sequential decision-making algorithms (\citep{jain2012differentially}, \citep{Mishra2015MAB}, \citep{tossou2016algorithms}). However, it has been proven that for contextual bandits under standard DP constraint, it is impossible to get sublinear regret bound  \citep{shariff2018differentially}. In RL, it is also impractical to try to recommend actions to users while protecting their information. Thus, some relaxed variants of standard DP should be considered under RL settings. Joint Differential Privacy (JDP) \citep{kearns2014mechanism}, which is weaker than standard DP but still provides strong privacy protection, ensures that even if an adversary can observe the recommended actions for all users except a given user, it is still statistically difficult to accurately identify the information from that user. Local Differential Privacy (LDP) \citep{duchi2013local} requires that each user’s raw data, which is encoded in the state-action-reward trajectory in RL setting, is privatized before being sent to the agent. LDP provides the privacy protection at the user level and ensures that even if the agent is unreliable, the privacy protection is still valid. Both notions have been well studied in bandits (\citep{shariff2018differentially}, \citep{basu2019differential}, \citep{zheng2020locally}) and draw more and more attention in RL (\citep{vietri2020private}, \citep{garcelon2021local}, \citep{Sayak2021}, \citep{pmlr-v206-qiao23a}). A general framework for designing optimistic private RL algorithms for value iteration was proposed in \citep{Sayak2021}, instantiated with suitable privatized counts to yield Private-UCB-VI, which satisfies both JDP and LDP. That framework was later improved in \citep{pmlr-v206-qiao23a}, yielding DP-UCBVI, another UCB-VI based algorithm that achieves the best known regret bound under JDP and LDP constraints to the best of our knowledge.

Despite these substantial advances, the landscape of privacy-preserving RL algorithms is still far from complete. Apart from the aforementioned variants of DP notions, some advanced DP variants, such as R\'enyi Differential Privacy (RDP) \citep{mironov2017renyi}, have not been introduced in RL settings so far. RDP provides a similar privacy guarantee and shares many important properties with standard DP while allowing for tighter analysis of composite mechanisms. Besides, existing studies focus mainly on value-based methods (such as Private-UCB-VI, DP-UCBVI), while little attention has been paid to algorithms based on randomized exploration and posterior sampling, such as Randomized Least Squares Value Iteration (RLSVI) \citep{Ian2019DeepExplorationviaRandomizedValueFunctions}. Compared to classic UCB-type methods, which usually require the design of a complex bonus term to realize optimism and are difficult to apply to many practical models such as neural network, algorithms based on randomized exploration have been more widely used in practice (\citep{chapelle2011empirical}, \citep{burda2018exploration}, \citep{osband2018RandomizedPriorFunctionsForDeepRL}). However, to the best of our knowledge, no differentially private version of RL algorithms based on randomized exploration has been proposed so far. In addition, there is no existing work on how the random noise in exploration-based RL algorithms influences privacy protection. This leaves a significant gap in understanding how randomized exploration interacts with the injected noise required by privacy mechanisms. 

In this work, we provide a privacy analysis that characterizes how the noise in RLSVI set for exploration serves a dual purpose by also providing privacy protection. More concretely, the randomization that RLSVI already injects for exploration purposes is sufficient to guarantee privacy, so no additional noise is needed. We formally prove that RLSVI is $(\varepsilon(\delta),\delta)$-joint differentially private in tabular MDP as is with $\varepsilon(\delta) = \frac{2AK}{H^2\log(2HSA)} + 2\sqrt{\frac{2AK\log(1/\delta)}{H^2\log(2HSA)}}$. 
\subsection{Related work}
For the family of randomized exploration methods in RL, RLSVI was first proposed in \citep{Ian2019DeepExplorationviaRandomizedValueFunctions} by injecting carefully chosen random noise into the value function to encourage exploration, with an expected regret guarantee provided. Subsequent works tightened the theory: Worst-case bounds for RLSVI were given in \citep{russo2019worstcaseregretboundsexploration}. Frequentist regret bounds for RLSVI were given in \citep{Zanette2020FrequentistRegretBoundsforRLSVI}. Worst-case bounds for RLSVI-type algorithms under tabular MDP were improved in \citep{agrawal2021improvedboundRLSVI}. Single Seed Randomization (SSR) was proposed in \citep{xiong2022SingleSeedRandomization} by applying certain clipping strategy to the value function and achieved the near-optimal minimax regret which also matches the best result in UCB-type algorithms up to logarithm terms.

Differential privacy issues have been studied widely under stochastic multi-arm bandit setting (\citep{Mishra2015MAB}, \citep{sajed2019mab}, \citep{azize2022bandits}, \citep{hu2022near}, \citep{tossou2016algorithms}). In addition, \citep{azize2023renyi} first introduced RDP \citep{mironov2017renyi} into finite-armed bandits, linear bandits and linear contextual bandits. They also designed corresponding algorithms satisfying RDP with valid regret bounds. In particular, \citep{ou2024_ThompsonSamplingItselfa} proved that the Thompson sampling algorithm with Gaussian prior is differentially private as is, which directly motivated our investigation of whether the same holds for RLSVI.

The following series of works focus on designing privacy-preserving algorithms under the episodic tabular MDP setting and all proved a regret guarantee for their algorithms under different privacy notions. \citep{vietri2020private} first introduced JDP \citep{kearns2014mechanism} into the
tabular MDP and proposed PUCB with JDP guarantee. In addition, they also presented lower bounds on sample complexity and regret bound under JDP constraints. 
\citep{garcelon2021local} introduced LDP \citep{duchi2013local} into the tabular MDP and they derived LDP-OBI with LDP guarantee. They also established a lower bound for regret minimization in finite-horizon MDPs under LDP constraints, which shows that a multiplicative effect of privacy cost on regret is unavoidable. After that, \citep{Sayak2021} proposed two general frameworks for designing optimistic private RL algorithms, one for policy optimization and another for value iteration. They instantiate these frameworks with suitable privatized counts and proposed Private-UCB-PO and Private-UCB-VI, which satisfy both JDP and LDP. By improving that framework, \citep{pmlr-v206-qiao23a} proposed a novel privatization of visitation numbers that satisfies several nice properties and designed another UCB-VI based algorithm: DP-UCBVI, which achieves the best regret bound under JDP and LDP constraints to the best of our knowledge. Furthermore, the "non-private" part of their regret bound matches the minimax lower bound in \citep{jin2018q}.

By comparing JDP and LDP, it is not hard to find that they provide privacy protection at different steps during the user-agent interaction. In JDP, privacy protection happens after users send their data to the (central) agent and thus the agent has access to all users’ raw data. The JDP guarantee remains meaningful only when we can trust the agent. In LDP, the protection occurs directly on the users' side but the fact that the agent is only given the perturbed data induces a higher privacy cost.  \citep{bai2026_NearOptimalReinforcementLearning} first introduced Shuffle Differential Privacy (SDP) \citep{cheu2019shuffleDP} into RL setting, which is an intermediate trust model stronger than the central DP model but with a lower privacy cost than the local model, making a good attempt to balance privacy protection and sample efficiency. They also proposed SDP-PE, the first RL algorithm under SDP with a valid regret guarantee and achieved the near-optimal bound.
Note that all of the aforementioned works focused on the single-agent RL setting. \citep{qiao2024differentially} first extended the definitions of JDP and LDP to
the multi-agent RL setting \citep{zhang2021multi}, where several agents simultaneously make decisions in an unfamiliar environment with the goal of maximizing their individual cumulative rewards. They proposed a general algorithm for DP multi-agent RL: DP-Nash-VI based on Nash-VI \citep{liu2021sharp}. And the regret bounds they derived under both JDP and LDP strictly generalize the best known results for single-agent RL with corresponding DP notions in \citep{pmlr-v206-qiao23a}. 

\section{Models and preliminaries}
We first introduce some notation. For $N\in\mathbb{Z}^+$, we define $\lrs{N}:=\lrc{1,2,\dots,N}$. $\lri{\cdot,\cdot}$ denotes the inner product. $\1(\cdot)$ denotes the indicator function. $\mathcal{N}(\mu,\sigma^2)$ denotes Gaussian distribution with mean $\mu$ and variance $\sigma^2$. For any set $V$, we use $\Delta(V)$ to denote the set of all possible probability distributions over $V$. 
\subsection{Definition and interaction protocol}
We consider finite-horizon \textit{Markov Decision Processes} (MDP) with non-stationary transitions, denoted by a tuple
\begin{equation*}
    M=\lr{\Ss,\A,H,\lrc{P_h}_{h=1}^H,\lrc{\mathcal{R}_h}_{h=1}^H,s_1}.
\end{equation*}
Here 
$\Ss$ and $\A$ are state and action spaces with $\lra{\Ss}=S$, $\lra{\A}=A$.
$H\in\N$ is a given finite-horizon.
$P_h:\Ss\times\A\to\Delta(\Ss)$ is the non-stationary transition kernel with $P_h(s'|s,a)$ representing the probability of transition from state $s$, action $a$ to next state $s'$ at time step $h$. It holds that $\sum_{s'\in\Ss}P_h(s'|s,a)=1$ for any $(h,s,a)$.
$\mathcal{R}_h:\Ss\times\A\to\Delta([0,1])$ denotes the corresponding distribution of reward. At state $s$, when taking action $a$, the reward obtained follows $\mathcal{R}_h(s,a)$. We use $r_h(s,a)\sim\mathcal{R}_h(s,a)$ to denote the random variable that follows the corresponding reward distribution. 
$s_1\in\Ss$ denotes a deterministic initial state.

A policy can be seen as a series of mappings $\pi=\lr{\pi_1,\dots,\pi_H}$, which prescribes what action to choose at various states and time steps. Based on the possible outputs, a policy could be \textit{deterministic}, where each $\pi_h:\Ss\to\A$ maps each state $s\in\Ss$ to an action directly. Or it could also be \textit{randomized}, where each $\pi_h$ maps each state $s\in\Ss$ to a probability distribution over actions, \textit{i.e.} $\pi_h:\Ss\to\Delta(\A),\ \forall h\in[H]$. In addition, we use $\Pi$ to denote the space of all possible policies.

\subsection{Value functions and optimality}

In this subsection, we introduce two useful concepts: state value function (also called state value or value function) and Q-value function (also called action-value function or Q-value). Given a policy $\pi$ and any $h\in[H]$, the value function $V_h^\pi(\cdot)$ and Q-value function $Q_h^\pi(\cdot,\cdot)$ associated with policy $\pi$ in the sub-episode consisting of periods $\lrc{h,\dots,H}$ are defined as: 
\begin{align*}
    V_h^\pi(s):=&\E_\pi\lrs{\sum_{t=h}^Hr_t\,\middle|\,  s_h=s},\ \forall s\in\Ss, \\Q_h^\pi(s,a):=&\E_\pi\lrs{\sum_{t=h}^Hr_t\,\middle|\,s_h=s,a_h=a},\ \forall s,a\in(\Ss,\A).
\end{align*} 
The state value $V_h^\pi(s)$ is the expected cumulative reward the agent receives from time $h$ to $H$ following policy $\pi$, starting from state $s$. It is related to both the MDP $M$ and the policy $\pi$. By definition, $V_h^\pi(s)=Q_h^\pi(s,\pi(s))$. In addition, as we assumed the reward $r_h$ is bounded in $[0,1]$, it holds trivially that for any $\pi$ and any $s,a$,
\begin{align*}
    0\leq V_h^\pi(s)\leq H-h+1 \\
    0\leq Q_h^\pi(s,a)\leq H-h+1
\end{align*}

In finite-horizon tabular MDPs, an optimal policy $\pi^\star$ always exists which maximizes $V_h^\pi(s)$ for all $s,h\in\Ss\times[H]$ simultaneously and we denote the value function and Q-value function with respect to $\pi^\star$, also known as the \textit{optimal value function} and the \textit{optimal Q-value function} respectively, by $V_h^\star(\cdot)$ and $Q_h^\star(\cdot,\cdot)$. We also introduce a vector representation: $V_h^\pi:=\lr{V_h^\pi(s_1),\dots,V_h^\pi(s_S)}\in\R^S$. To simplify many expressions, we set $V_{H+1}^\pi\equiv0\in\R^S,\ \forall \pi\in\Pi$.
Then \textit{Bellman equation} follows $\forall h,s,a\in[H]\times\Ss\times\A$,
\begin{align*}
    Q_h^\pi(s,a)=&\E[r_h(s,a)]+\E_{x\sim P_h(\cdot|s,a)}[V_{h+1}^\pi(x)] \\
    =&R_h(s,a)+\lri{P_h(\cdot|s,a),V_{h+1}^\pi},\\ V_h^\pi(s)=&\E_{a\sim\pi_h}[Q_h^\pi(s,a)].
\end{align*}


\subsection{Episodic reinforcement learning}


In (online) episodic RL, an agent interacts
with an environment with unknown transition kernel and reward distribution across episodes of fixed length $H$, and the initial state of the agent is reset whenever a new episode begins. More concretely, let $K$ be the total number of episodes that the agent plays and the number of steps is $T:=KH$. The $K$ trajectories generated during the interaction take the form:
\begin{align*}
    (s_1^1,a_1^1,r_1^1,\dots&,s_H^1,a_H^1,r_H^1,s_{H+1}^1) \\
    (s_1^2,a_1^2,r_1^2,\dots&,s_H^2,a_H^2,r_H^2,s_{H+1}^2)\\
    &\cdots \\
    (s_1^K,a_1^K,r_1^K,\dots&,s_H^K,a_H^K,r_H^K,s_{H+1}^K)
\end{align*}
They are generated by the following rule: 
\begin{equation*}
    a_h^k\sim\pi_h^k(\cdot|s_h),\ r_h^k\sim\mathcal{R}_h(s_h,a_h),\ s_{h+1}^k\sim P_h(\cdot|s_h,a_h),\ \forall h\in[H],\ \forall k\in[K].
\end{equation*}
For convenience, we assume that the initial state is fixed for every episode, but similar conclusions could also be drawn assuming that the initial state is drawn from some distribution over $\Ss$. We also define the history $\HH_k$ as the collection of the first $k$ trajectories. 
We measure the performance of an online reinforcement learning algorithm by its (expected) regret which is defined as
\begin{equation*}
    \Regret(K):=\sum_{k=1}^K\left[V_1^\star(s_1^k)-V_1^{\pi^k}(s_1^k)\right],
\end{equation*}
where $s_1^k$ is the initial state and $\pi^k=\left(\pi_1^k,\dots,\pi_H^k\right)$ is the policy executed at episode $k$, $\forall k\in[K]$. 


\subsection{Differential privacy under episodic RL}
In this section, we discuss the basic concepts of differential privacy in an episodic RL setting and the basic mechanisms to achieve DP and its variants. We first reconsider the RL protocol from a different perspective: from agent-environment interaction to user-agent interaction. Specifically, during the $h$-th step of the $k$-th episode, user $u_k$ sends its state $s_h^k$ to the agent $\M$, at which point $\M$ sends back an action $a_h^k$, and finally $u_k$ sends its reward $r_h^k$ to $\M$. Formally, we denote the sequence of $K$ users who participate in the above RL protocol by $U_K=(u_1,\dots,u_K)\in\U^K$, where $\mathcal{U}$ is the space of all such possible users. We let 
\begin{equation*}
    \M(U_K)=(a_1^1,\dots,a_H^1,a_1^2,\dots ,a_H^2,\dots,a_1^K,\dots,a_H^K)\in\A^{KH}
\end{equation*}
denote the whole sequence of actions chosen by agent $\M$. 


\begin{definition}[Differential Privacy (DP) \citep{dwork2006calibrating}]
    For any $\varepsilon>0$ and $\delta\in[0,1)$, a mechanism $\M:\U^K\to\A^{KH}$ is $(\varepsilon,\delta)$-differentially private if for any possible user sequences $U_K$ and $U'_K$ differing on a single user and any subset $E$ of $\A^{KH}$,
    \begin{equation*}
        \P[\M(U_K)\in E]\leq e^\varepsilon\P\lrs{\M(U'_K)\in E}+\delta
    \end{equation*}
    If $\delta=0$, we say that $\M$ is $\varepsilon$-differentially private ($\varepsilon$-DP).
\end{definition}
\begin{remark}
\label{Neighboring choice}
    In this paper, we assume the states and actions are the same for two different users, which implies that $U_K$ and $U_K'$ above could differ on all or some of the rewards collected in one specific trajectory. While this is a strong assumption and it may not hold in all practical scenarios, it allows us to focus on the core technical challenges. Relaxing it is left for future work.
\end{remark}


\begin{definition}[R\'{e}nyi Differential Privacy (RDP) \citep{mironov2017renyi}]
    For any $\varepsilon>0$ and $\alpha>1$, a mechanism $\M:\U^K\to\A^{KH}$ is ($\alpha,\varepsilon$)-RDP if for any possible user sequences $U_K$ and $U'_K$ differing on a single user, it holds that
    \begin{equation*}
        D_\alpha\left(P_{\M(U_K)}\,\middle\|\,P_{\M(U_K')}\right)\leq\varepsilon
    \end{equation*}
    where $P_{\M(U_K)}$ and $P_{\M(U_K')}$ are the probability distributions of $\M(U_K)$ and $\M(U_K')$.
    \begin{equation*}
        D_\alpha(P\,\|\,Q):=\frac{1}{\alpha-1}\log\E_{x\sim Q}\left(\frac{P(x)}{Q(x)}\right)^\alpha
    \end{equation*}
    is the R\'enyi divergence of order $\alpha$ for $P$ and $Q$.
\end{definition}
Here, $\alpha$ is a parameter that controls how the difference between $P$ and $Q$ is measured. As $\alpha$ increases, the divergence becomes more sensitive to larger differences in the ratios $\frac{P(x)}{Q(x)}$. And a smaller $\alpha$ provides a more balanced consideration of the entire distribution. 
For the rest of this paper, we omit the notation of probability distributions in R\'{e}nyi divergence and only write the notation of mechanisms, \textit{i.e.} $D_\alpha\left(\M(U_K)\,\|\,\M(U_K')\right):=D_\alpha\left(P_{\M(U_K)}\,\middle\|\,P_{\M(U_K')}\right)$.

Both DP and RDP require that the mechanism take joint outputs of all $K$ users into consideration at the same time, which could be impractical in episodic RL setting. Let us take a DP mechanism $\M$ for example. For user sequences $U_K$ and $U_k'$ differing on the $k$-th user for some $k\in[K]$, the two action sequences suggested for the two different users should be indistinguishable as part of $\M(U_K)$ and $\M(U_K')$. However, this also implies an inefficient mechanism because it could not react to any user's personal features, which motivates some other variants of DP to be introduced.

The goal of Joint Differential Privacy (JDP) is to protect the privacy of individual users in scenarios where the mechanism generates joint outputs that involve multiple users. Unlike DP or RDP, JDP only ensures that changes to one user's data do not significantly affect the joint output provided to other users.
\begin{definition}[Joint Differential Privacy (JDP) \citep{kearns2014mechanism}]
    For any $\varepsilon>0$, a mechanism $\M:\U^K\to\A^{KH}$ is $(\varepsilon,\delta)$-joint differentially private if for all $k\in[K]$, for all user sequences $U_K,U'_K$ differing only on the $k$-th user and for all sets of actions $\A_{-k}\subset\A^{(K-1)H}$ given to all but the $k$-th user,
    \begin{equation*}
        \P\left(\M_{-k}(U_K)\in\A_{-k}\right)\leq e^\varepsilon\P[\M_{-k}(U'_K)\in\A_{-k}]+\delta
    \end{equation*}
    where $\M_{-k}$ denotes the mechanism that excludes the $k$-th user's input and computes joint outputs for all other users, \textit{i.e.} $\M_{-k}(U_K):=\M(U_K)\setminus\lrc{a_h^k}_{h=1}^H$.
\end{definition}
For the interaction protocol we mentioned above, the users need to send their raw data (i.e. states and rewards) to the agent. Therefore, a JDP mechanism could only provide privacy protection under the assumption that the agent is trusted.

\begin{definition}[$\ell_1$-sensitivity]
    The $\ell_1$-sensitivity of a query function $f:\mathcal{D}\to\R$ is 
    \begin{equation*}
        \Delta_1 f=\max_{D,D'}\lrv{f(D)-f(D')}
    \end{equation*}
    where $\mathcal D$ denotes the space of all possible datasets. The maximum is taken over all \textit{neighboring} datasets $D,D'$, which means they differ on exactly one element. 
\end{definition}
\begin{lemma}[Gaussian mechanism preserves $(\alpha,\varepsilon)$-RDP]
    \label{lemma: parameters for Gaussian mech to realize RDP}
     Let $f:\D \to \R$ be a query function with $\ell_1$ sensitivity of $\Delta$. The  Gaussian mechanism $G_\sigma$ based on $f$ defined as
     \begin{equation*}
         G_\sigma f(D):=f(D)+\mathcal{N}(0,\sigma^2)
     \end{equation*}
     satisfies
    \begin{equation*}
        D_\alpha(G_\sigma(D)\,\|\,G_\sigma(D'))\leq\frac{\alpha\Delta^2}{2\sigma^2}
    \end{equation*}
    Specifically, $G_{\frac{\alpha\Delta^2}{2\varepsilon}}f$ satisfies $(\alpha,\varepsilon)$-RDP.
\end{lemma}


The following lemma shows that RDP could be preserved under adaptive sequential composition.
\begin{lemma}[Composition for RDP mechanism, \citep{mironov2017renyi}]
\label{Composition for RDP mechanism}
    Let $f:\mathcal D\to \mathcal R_1$ be ${(\alpha,\varepsilon)}$-RDP and $g:\mathcal R_1\times \mathcal D\to\mathcal R_2$ be $(\alpha,\varepsilon_2)$-RDP for any $x\in\mathcal{R}_1$. Then the mechanism defined as $(X,Y)$, where $X\leftarrow f(D)$ and $Y\leftarrow g(X,D)$, satisfies $(\alpha,\varepsilon_1+\varepsilon_2)$-RDP.
\end{lemma}

The next lemma shows that RDP could be converted to $(\varepsilon,\delta)$-DP.
\begin{lemma}[From RDP to DP, \citep{mironov2017renyi}]
\label{From RDP to DP}
    If $f$ is an $(\alpha,\varepsilon)$-RDP mechanism, it also satisfies $(\varepsilon+\frac{\log (1/\delta)}{\alpha-1},\delta)$-DP for any $\delta\in(0,1)$.
\end{lemma}
In the end, we introduce an important lemma for designing a JDP algorithm.
\begin{lemma}[Billboard Lemma, Lemma 9 from      \citep{hsu2014_PrivateMatchingsAllocations}]
    \label{Billboard Lemma}
    Let $\M':\U^K\to\YY$ be a mechanism satisfying $(\varepsilon,\delta)$-DP. Consider any set of functions $f_i:\U\times \YY\to \A^{H}$ that are independent of the users' data, then the composition $\lrc{f_i\left(\Pi_iU^K,\M'(U^K)\right)}$ is $(\varepsilon,\delta)$-JDP, where $\Pi_i:\U^K\to\U$ is the projection from the entire user sequence to the $i$-th user. 
\end{lemma}
Here $\M'$ could be viewed as a mechanism to produce DP statistics used for the algorithm, while $\lrc{f_i\left(\Pi_iU^K,\M'(U^K)\right)}$ could be viewed as the algorithm. 
The Billboard lemma shows that if the output of one algorithm is a composition of several individual users' data and the output generated by a $(\varepsilon,\delta)$-DP mechanism, then the algorithm satisfies $(\varepsilon,\delta)$-JDP.

\section{RLSVI is joint differentially private}
In this section, we show that RLSVI for tabular MDP (Algorithm \ref{Algorithm: RLSVI}) is a JDP algorithm. 
To prove this formally, we only need to show that the release of policy $\lrc{\pi^k}_{k=1}^K$ satisfies the constraint of $\lr{\varepsilon,\delta}$-DP due to the Lemma \ref{Billboard Lemma}. Informally, $\lrc{\pi^k}_{k=1}^K$ is a post-processing of the release of $\lrc{\tilde{Q}_h^k(s,a)}_{k,h,s,a}$, where a Gaussian noise is added to every single $\tilde{Q}_h^k(s,a)$. Here, the difficulty lies in the fact that $\tilde{Q}_h^k(s,a)$ is not independent of each other and the noise added to $\tilde{Q}_h^k(s,a)$ could influence all the $\tilde{Q}_t^k(s,a)$, $\forall t\in[h-1]$. We will first rephrase RLSVI in the language of RDP mechanisms and then prove that for a mechanism at a single step, it satisfies RDP constraint (Lemma \ref{Single step privacy guarantee}). After that we will show the privacy guarantee maintains after the composition (Lemma \ref{pi^k is RDP}) and finally the algorithm is JDP (Theorem \ref{Theorem:RLSVI is JDP}).

\paragraph{Notation for empirical estimates} We use the notation of counts as 
\begin{align*}
    N_h^k(s,a)=&\sum_{i=1}^{k-1}\1 \lrc{s_h^{i}=s, a_h^{i}=a} \\
    N_h^k(s,a,s')=&\sum_{i=1}^{k-1}\1\lrc{s_h^{i}=s, a_h^{i}=a,s_{h+1}^{i}=s'} \\
\end{align*}
And in language of reward, we use
\begin{equation*}
    R_h^k(s,a):=\sum_{i=1}^{k-1}\1(s_h^i=s,a_h^i=a)\cdot r_h^i
\end{equation*}
to denote the cumulative reward at $(h,s,a)$ before the $k$-episode.

We also define:
\begin{equation}
    \label{Def: empirical average reward and empirical transition kernel}
    \begin{aligned}
        \hat{R}_h^k(s,a)=&\frac{R_h^k(s,a)}{N_h^k(s,a)} \\
        \hat{P}_h^k(s'|s,a)=&\frac{N_h^k(s,a,s')}{N_h^k(s,a)}
    \end{aligned}
\end{equation}
Specifically, we define $\hat{R}_h^k(s,a)=\hat{P}_h^k(s'|s,a)=0$, when $N_h^k(s,a)=0$.

\paragraph{Implementation details}
Informally, RLSVI shares basically the same structure as the Bellman equations based on empirical MDP. The difference is that it changes the Q-estimates by adding a carefully chosen random noise, which has been proven useful for deep exploration.  Specifically, a greedy policy $\pi^k$ is obtained directly to maximize the $Q$-estimates at each episode $k$. By executing the policy $\pi^k$, a trajectory is generated to update the empirical counts and rewards, and then a new episode begins. In particular, for the first episode, $\tilde Q_h^1(s,a)$ depend only on the random noise for any $h,s,a$ and could be viewed as a prior sample from Gaussian distribution.
\begin{algorithm}[H]
\label{Algorithm: RLSVI}
\caption{RLSVI for tabular MDP}
\KwIn{$K$, $H$, $S,A$, tuning parameters $\lrc{\beta_k}_{k\in\N}$, where $\beta_k=\frac{1}{2}SH^3\log (2HSAk)$.}

\textbf{Initialize:} Empirical counts ${R}^1_h(s,a) = {N}^1_h(s,a) = {N}^1_h(s,a,s') = 0$ for all $(h,s,a,s') \in [H]\times \Ss \times \A \times \Ss$.\;

\For{$k=1,2,\dots,K$}{
    Initialize value functions: $\tilde{V}_{H+1}^k(s) = 0$, $\forall s$,
    
    \For{$h=H,H-1,\dots,1$}{
        Compute $\hat{P}^k_h(s'|s,a)$ and $\hat{R}^k_h(s,a)$ as in (\ref{Def: empirical average reward and empirical transition kernel})\;
        
        \For{$(s,a)\in \Ss\times \A$}{
            $\tilde{Q}_h^k(s,a) =   \hat{R}^k_h(s,a) + \sum_{s'\in\Ss} \hat{P}^k_h(s'|s,a)\cdot \tilde{V}^k_{h+1}(s') + w_h^k(s,a)$, where $w_h^k(s,a)\sim\mathcal{N}(0,\frac{\beta_k}{N_h^k(s,a)+1})$\;
        }
    
        \For{$s \in S$}{
            $\tilde{V}^k_h(s) = \max_{a\in A} \tilde{Q}^k_h(s,a)$\;
            $\pi^k_h(s) = \argmax_{a\in A} \tilde{Q}^k_h(s,a)$ \;
        }
    }
    Roll out a trajectory $(s^k_1,a^k_1,r^k_1,\dots,s^k_{H+1})$ by executing the policy $\pi^k=(\pi_h^k)_{h=1}^H$\;
    
    Update the empirical counts ${R}^{k+1}_h(s,a),\ {N}^{k+1}_h(s,a),\ {N}^{k+1}_h(s,a,s')$\;
}
\end{algorithm}

\begin{remark}
    We could also formalize the RLSVI in another equivalent way. At the beginning of each episode $k$, we construct a perturbed empirical MDP $\tilde{M}^k=(\Ss,\A,\hat{P}_h^k,\lrc{\hat{R}^k+w^k}_h,s_1)$ and then solve the optimal policy $\pi^k$ for it.
\end{remark}

Now we review RLSVI in the language of DP mechanism. For each $k$, we define 
\begin{align*}
    \M_h^k(\HH_{k-1},\M_{h+1}^k):=&\lrc{\tilde{Q}_h^k(s,a)}_{s,a} \\
    =&\lrc{\hat{R}^k_h(s,a) + \sum_{s'\in\Ss} \hat{P}^k_h(s'|s,a)\cdot \tilde{V}^k_{h+1}(s') + w_h^k(s,a)}_{s,a}\in\R^{SA}, \forall \ h\in[H-1]. \\
    \text{Specifically, }\M_H^k(\HH_{k-1}):=&\lrc{\hat{R}^k_H(s,a) + w_H^k(s,a)}_{s,a}\in\R^{SA},\  \M_{H+1}^k(\HH_{k-1}):=\vec 0\in \R^{SA}. \\
\end{align*}

The following lemma shows the privacy guarantee of the aforementioned mechanisms. Before the proof, it may seem counterintuitive for $H^3$ to appear in the denominator, which implies that as length of each episode increases, the privacy protection becomes stronger. We point out that it is due to the design of variance parameter $\beta_k=O(SH^3)$, which implies larger noise as $H$ increases.
\begin{lemma}
\label{Single step privacy guarantee}
    The mechanism $\M_h^k(\HH_{k-1},\M_{h+1}^k)$ satisfies $(\alpha,\frac{2\alpha A}{H^3\log(2HSA)})$-RDP with respect to observed rewards for any $k\in[K],\ h\in[H]$.  
\end{lemma}

\begin{proof}
    Note that $\tilde V_{h+1}^k(s')$ is the only quantity that depends on the output of $\M_{h+1}^k$ and is also a deterministic function for any $s'\in\Ss$. Once the output of $\M_{h+1}^k$ and the history of previous data are given, the mechanism $\M_h^k$ could be viewed as a Gaussian mechanism. More formally, let us define
    \begin{align*}
        f_{h,s,a}^k(\HH_{k-1},\M^k_{h+1}):=\hat{R}^k_h(s,a) + \sum_{s'\in\Ss} \hat{P}^k_h(s'|s,a)\cdot \tilde{V}^k_{h+1}(s'), \\
        \M_{h,s,a}^k(\HH_{k-1},\M_{h}^k):=\tilde Q_h^k(s,a)= f_{h,s,a}^k(\HH_{k-1},\M^k_{h+1})+w_h^k(s,a),
    \end{align*}
    where $w_h^k(s,a)\sim\mathcal{N}(0,\frac{\beta_k}{N_h^k(s,a)+1})$.
    The $\ell_1$ sensitivity of $f_{h,s,a}^k$ could be defined as 
    \begin{equation}
    \label{eq:sensitivity of f hsak}
        \Delta f_{h,s,a}^k:=\max_{\vec{c}\in\R^{SA}} \max_{\HH_{k-1},\HH_{k-1}'}\lrv{f_{h,s,a}^k(\HH_{k-1},\vec{c})-f_{h,s,a}^k(\HH_{k-1}',\vec{c})}_1, \forall \ k,h,s,a.
    \end{equation}
    where the first maximum is taken over a fixed $\vec c$ and all possible neighboring histories $\HH_{k-1},\HH_{k-1}'$ that are different only on the rewards collected in one episode. The second maximum is taken over all possible values of $\vec c$. 
    Consider two neighboring histories $\HH_{k-1}$, $\HH_{k-1}'$ and their corresponding empirical rewards $\hat{R}_h^k(s,a)$, $\hat{R}_h^k(s,a)'$. Recall that
    \begin{equation*}
        \hat R_h^k(s,a):=\frac{1}{N_h^k(s,a)}\sum_{i=1}^{k-1}\1(s_h^i=s,a_h^i=a)\cdot r_h^i
    \end{equation*}
    Thus,
    \begin{align*}
        \max_{\HH_{k-1},\HH_{k-1}'} \lra{\hat R_h^k(s,a)-\hat R_h^k(s,a)'} \leq \frac{1}{N_h^k(s,a)}
    \end{align*}
    In addition, by our model choice of neighboring datasets, $\hat{P}_h^k(s'|s,a)$ stays the same when we change one trajectory for any $s',s,a$. 
    And for a fixed output of $\M_{h+1}^k$ which is a table of perturbed Q-values $\lrc{\tilde Q_{h+1}^k(s,a)}_{s,a}$, $\tilde{V}^k_h(s') = \max_{a\in A} \tilde{Q}^k_h(s',a)$ will not change by replacing the rewards in one trajectory for any $s'$. Therefore, we have 
    \begin{align*}
        \Delta f_{h,s,a}^k=\max_{\vec{c}\in\R^{SA}} \max_{\HH_{k-1},\HH_{k-1}'}\lra{\hat R_h^k(s,a)-\hat R_h^k(s,a)'} \leq \frac{1}{N_h^k(s,a)} \\
    \end{align*}
    As the variance of the noise $w_h^k(s,a)$ is $\frac{\beta_k}{N_h^k(s,a)+1}$, by Lemma \ref{lemma: parameters for Gaussian mech to realize RDP} we have
    \begin{align*}
    &D_\alpha(\M_{h,s,a}^k(\HH_{k-1},\M_{h+1}^k)\,\|\,\M_{h,s,a}^k(\HH_{k-1}',\M_{h+1}^k)) \\
        =&\frac{\alpha\lr{N_h^k(s,a)+1}}{2\beta_k{N_h^k(s,a)^2}} 
        \leq\frac{\alpha}{\beta_kN_h^k(s,a)}\leq\frac{2\alpha}{SH^3\log(2HSAk)} \\
        \leq&\frac{2\alpha}{SH^3\log(2HSA)}
    \end{align*}
    where we assume that $N_h^k(s,a)\geq1$, otherwise there is no data to protect for that $(k,h,s,a)$ pair and no privacy loss. The first inequality holds by $\frac{N_h^k(s,a)+1}{N_h^k(s,a)}\leq2$ under assumption. The last inequality holds by $k\geq1$. By Lemma \ref{Composition for RDP mechanism}, we have $\M_h^k(\HH_{k-1},\M_{h+1}^k)$ satisfies $\lr{\alpha,\frac{2\alpha A}{H^3\log(2HSA)}}$-RDP.
\end{proof}

By applying the composition lemma for RDP, we have the following result.
\begin{lemma}
    \label{pi^k is RDP}
    The release of $\lrc{\pi^k}_{k=1}^K$ in Algorithm \ref{Algorithm: RLSVI} satisfies $\lr{\alpha,\frac{2\alpha AK}{H^2\log(2HSA)}}$-RDP with respect to observed rewards.
\end{lemma}
\begin{proof}
    We first note that for a fixed $k$, the adaptive composition of $\lrc{\M_h^k(\HH_{k-1},\M_{h+1}^k)}_h=:\M^k(\HH_{k-1})$ satisfies $\lr{\alpha,\frac{2\alpha A}{H^2\log(2HSA)}}$-RDP by Lemma \ref{Composition for RDP mechanism}. By applying Lemma \ref{Composition for RDP mechanism} again with respect to $k$ we have that $\lrc{\M^k(\HH_{k-1})}_k$ satisfies $\lr{\alpha,\frac{2\alpha AK}{H^2\log(2HSA)}}$-RDP. Since $\lrc{\pi^k}_{k=1}^K$ is a post-processing of $\lrc{\M^k(\HH_{k-1})}_k$, it is proved that the release of $\lrc{\pi^k}_{k=1}^K$ satisfies $\lr{\alpha,\frac{2\alpha AK}{H^2\log(2HSA)}}$-RDP.
\end{proof}

Finally, we convert RDP guarantee to $(\varepsilon,\delta)$-DP and prove the JDP guarantee of the algorithm. 
\begin{theorem}
\label{Theorem:RLSVI is JDP}
    RLSVI in tabular setting as Algorithm \ref{Algorithm: RLSVI} satisfies $(\varepsilon(\alpha,\delta),\delta)$-JDP, where $\varepsilon(\alpha,\delta)=\frac{2\alpha AK}{H^2\log(2HSA)}+\frac{\log(1/\delta)}{\alpha-1}$.
\end{theorem}
\begin{proof}
    By Lemma \ref{From RDP to DP} and Lemma \ref{pi^k is RDP}, we have that $\lrc{\pi^k}_{k=1}^K$ satisfies $\lr{\frac{2\alpha AK}{H^2\log(2HSA)}+\frac{\log(1/\delta)}{\alpha-1},\delta}$-DP. We notice that in each episode $k$, the action series produced by the algorithm only depends on $\pi^k$ and the data collected in that episode. By Lemma \ref{Billboard Lemma}, Algorithm \ref{Algorithm: RLSVI} is $\lr{\frac{2\alpha AK}{H^2\log(2HSA)}+\frac{\log(1/\delta)}{\alpha-1},\delta}$-JDP.
\end{proof}
It is worth noting that Theorem \ref{Theorem:RLSVI is JDP} holds for any $\alpha>1$, which means we can optimize this result according to $\alpha$ to get the tightest $(\varepsilon(\delta), \delta)$-JDP.

\begin{corollary}
\label{corollary: RLSVI is JDP after optimizing alpha}
    RLSVI in tabular setting as Algorithm \ref{Algorithm: RLSVI} satisfies $(\varepsilon(\delta),\delta)$-JDP, where $\varepsilon(\delta)=
\frac{2AK}{H^2\log(2HSA)} + 2\sqrt{\frac{2AK\log(1/\delta)}{H^2\log(2HSA)}}$.
\end{corollary}


\section{Discussion}
In this work, we provide a privacy analysis of RLSVI in episodic tabular MDP on how the inherent noise injected for deep exploration also guarantees privacy. We show that the policy released in RLSVI could be rephrased as a post-processing of a composition of RDP mechanisms and RLSVI satisfies $(\varepsilon(\delta),\delta)$-JDP. As no extra noise is added, the performance guarantees (i.e. Bayesian regret bound \citep{Ian2019DeepExplorationviaRandomizedValueFunctions} and worst-case regret bound \citep{russo2019worstcaseregretboundsexploration}) for RLSVI stay without any privacy cost.

In our choice of model for neighboring datasets (Remark \ref{Neighboring choice}), we assumed that the user states remain unchanged when replacing a trajectory, which was made to simplify the calculation of the sensitivity in Equation \ref{eq:sensitivity of f hsak}. However, the user states could also contain sensitive information. To relax this assumption, another assumption on the value of $\tilde Q_h^k(s,a)$ becomes necessary, i.e. $\tilde Q_h^k(s,a)\leq c$ for some $c\in \R$. Otherwise the sensitivity of $f_{h,s,a}^k$ could be unbounded. It is possible and very interesting to try combining the clipping techniques used in \citep{agrawal2021improvedboundRLSVI} and \citep{xiong2022SingleSeedRandomization}, which is also beneficial for regret analysis.

For future work, our goal is to extend our privacy analysis framework to be applicable to other RL algorithms based on randomized exploration. In addition, empirical comparison with other differentially private RL algorithms could also be added.

\section*{Acknowledgment}
The work of Pratik Gajane and Haiyang Lu was supported by the French National Research Agency (ANR) under grant ANR-24-CPJ1-0088-01. The funders had no role in the study design, analysis, the conclusions drawn, or the writing of this manuscript.

\bibliographystyle{abbrv}
\bibliography{references}  
\appendix

\section{Technical Appendices and Supplementary Material}
\begin{proof}[Proof of Lemma \ref{lemma: parameters for Gaussian mech to realize RDP}]
    \begin{align*}
         D_\alpha(G_\sigma(D)\,\|\,G_\sigma(D'))=& D_\alpha(\mathcal{N}(f(D),\sigma^2)\,\|\,\mathcal{N}(f(D'),\sigma^2))\\
         =&\frac{\alpha(f(D)-f(D'))^2}{2\sigma^2} \\
         \leq&\frac{\alpha\Delta^2}{2\sigma^2}
    \end{align*}
    The second equality holds by $D_\alpha(\mathcal{N}(\mu_1,\sigma^2)\,\|\,\mathcal{N}(\mu_2,\sigma^2))=\frac{\alpha(\mu_1-\mu_2)^2}{2\sigma^2},\ \forall\mu_1,\mu_2\in\R,\ \forall\sigma^2>0$ (Equation 10 from \citep{van2014renyi}), 
\end{proof}

\begin{proof}[Proof of Corollary \ref{corollary: RLSVI is JDP after optimizing alpha}]
    Assuming $\delta$ to be a constant, the derivative of $\varepsilon(\alpha,\delta)$ with respect to $\alpha$ is 
    \begin{align*}
        \frac{\partial\varepsilon(\alpha,\delta)}{\partial\alpha}=\frac{2 AK}{H^2\log(2HSA)}-\frac{\log(1/\delta)}{(\alpha-1)^2}
    \end{align*}
    It is an increasing function and by setting it to be zero, $\varepsilon$ achieves the minimum value.
    \begin{equation*}
        \alpha=H\sqrt{\frac{\log(2HSA)\log(1/\delta)}{2AK}}+1
    \end{equation*}
    And correspondingly we have $\varepsilon(\delta)=
\frac{2AK}{H^2\log(2HSA)} + 2\sqrt{\frac{2AK\log(1/\delta)}{H^2\log(2HSA)}}$.
\end{proof}

\end{document}